\newcount\Comments  
\documentclass{ourstyle}
\usepackage[dvips]{color}

\usepackage{latexsym, epic, eepic}
\usepackage{amsmath}
\usepackage{amssymb}
\usepackage{mathrsfs}
\usepackage{arydshln}
\usepackage[square,comma,numbers]{natbib}
\usepackage{url}

\newcommand{\ignore}[1]{}
\definecolor{darkgreen}{rgb}{0,0.5,0}
\definecolor{darkred}{rgb}{0.7,0,0}

\newcommand{\maximize}{\mathop{\rm maximize}}

\newcommand{\real}{\mathbb{R}}

\newcommand{\angles}[1]{\left\langle #1 \right\rangle}
\newcommand{\abs}[1]{\left| #1 \right|}

\newtheorem{theorem}{Theorem}
\newtheorem{definition}{Definition}
\newtheorem{lemma}{Lemma}
\newtheorem{corollary}{Corollary}
\renewcommand{\qed}{\hfill $\framebox(6,6){}$}

\newcommand{\q}{{\vec{q}}}
\renewcommand{\r}{{\vec{r}}}
\newcommand{\p}{{\vec{p}}}
\renewcommand{\u}{{\vec{u}}}

\newcommand{\e}{{\textrm{e}}}
\newcommand{\loss}{\ell}
\newcommand{\Loss}{{L}}
\newcommand{\alg}{{\mathcal{A}}}

\newcommand{\risk}{\rho}
\newcommand{\entro}{{\textrm{H}}}
\newcommand{\Reg}{{\mathcal{R}}}
\newcommand{\simp}{{\Delta}}

\newcommand{\squishlist}{
   \begin{list}{$\bullet$}
    { \setlength{\itemsep}{0pt}      \setlength{\parsep}{3pt}
      \setlength{\topsep}{3pt}       \setlength{\partopsep}{0pt}
      \setlength{\leftmargin}{1.5em} \setlength{\labelwidth}{1em}
      \setlength{\labelsep}{0.5em} } }
\newcommand{\squishend}{  \end{list}  }

\title{A New Understanding of Prediction Markets \\
Via No-Regret Learning}

\numberofauthors{2}

\author{
\alignauthor
Yiling Chen\\
\affaddr{School of Engineering and Applied Sciences}\\
\affaddr{Harvard University}\\
\affaddr{Cambridge, MA 02138}\\
\email{yiling@eecs.harvard.edu}
\alignauthor
Jennifer Wortman Vaughan \\
\affaddr{School of Engineering and Applied Sciences}\\
\affaddr{Harvard University}\\
\affaddr{Cambridge, MA 02138}\\
\email{jenn@seas.harvard.edu}
}

\begin{document}

\maketitle

\abstract{We explore the striking mathematical connections that exist between market scoring rules, cost function based prediction markets, and no-regret learning.  We show that any cost function based prediction market can be interpreted as an algorithm for the commonly studied problem of learning from expert advice by equating trades made in the market with losses observed by the learning algorithm.  If the loss of the market organizer is bounded, this bound can be used to derive an $O(\sqrt{T})$ regret bound for the corresponding learning algorithm.  We then show that the class of markets with convex cost functions exactly corresponds to the class of Follow the Regularized Leader learning algorithms, with the choice of a cost function in the market corresponding to the choice of a regularizer in the learning problem.  Finally, we show an equivalence between market scoring rules and prediction markets with convex cost functions.  This implies that market scoring rules can also be interpreted naturally as Follow the Regularized Leader algorithms, and may be of independent interest.  These connections provide new insight into how it is that commonly studied markets, such as the Logarithmic Market Scoring Rule, can aggregate opinions into accurate estimates of the likelihood of future events.}








\section{Introduction}

Imagine you are interested in learning an accurate estimate of the probability that the United States unemployment rate for a particular month will fall below 10\%.  You could choose to spend hours digging through news articles, reading financial reports, and weighing various opinions against each other, eventually coming up with a reasonably informed estimate.  However, you could potentially save yourself a lot of hassle (and obtain a better estimate!) by appealing to the wisdom of crowds.

A \emph{prediction market} is a financial market designed for information aggregation.  
%
%
For example, in a cost function based prediction market~\cite{CP07}, the organizer (or \emph{market maker}) trades a set of securities corresponding to each potential outcome of an event.  The market maker might offer a security that pays \$1 if and only if the United States unemployment rate for January 2010 is above 10\%.  A risk neutral trader who believes that the true probability that the unemployment rate will be above 10\% is $p$ should be willing to buy a share of this security at any price below $\$p$.  Similarly, he should be willing to sell a share of this security at any price above $\$p$.  For this reason, the current market price of this security can be viewed as the population's collective estimate of how likely it is that the unemployment rate will be above 10\%.

These estimates have proved quite accurate in practice in a wide variety of domains.  (See \citet{LHI09} for an impressive assortment of examples.)  The theory of rational expectations equilibria offers some insight into why prediction markets in general should converge to accurate prices, but is plagued by strong assumptions and no-trade theorems~\cite{PS07}.  Furthermore, this theory says nothing of why particular prediction market mechanisms, such as Hanson's increasingly popular Logarithmic Market Scoring Rule (LMSR) ~\cite{H03,H07}, might produce more accurate estimates than others in practice.  In this work, we aim to provide additional insight into the learning power of particular market mechanisms by highlighting the deep mathematical connections between prediction markets and no-regret learning.

It should come as no surprise that there is a connection between prediction markets and learning.  The theories of markets and learning are built upon many of the same fundamental concepts, such as proper scoring rules (called proper losses in the learning community) and Bregman divergences.  To our knowledge, \citet{CFLPW08} were the first to formally demonstrate a connection, showing that the standard Randomized Weighted Majority regret bound~\cite{FS97} can be used as a starting point to rederive the well-known bound on the worst-case loss of a LMSR marker maker.  (They went on to show that PermELearn, an extension of Weighted Majority to permutation learning~\cite{HW09}, can be used to efficiently run LMSR over combinatorial outcome spaces for betting on rankings.)  As we show in Section~\ref{sec:connection}, the converse is also true; the Weighted Majority regret bound can be derived directly from the bound on the worst-case loss of a market maker using LMSR.  However, the connection goes much deeper.

In Section~\ref{sec:connection}, we show how \emph{any} cost function based prediction market with bounded loss can be interpreted as a no-regret learning algorithm.  Furthermore, if the loss of the market maker is bounded, this bound can be used to derive an $O(\sqrt{T})$ regret bound for the corresponding learning algorithm. The key ides is to view the \emph{trades} made in the market as \emph{losses} observed by the learning algorithm.  We can then think of the market maker as learning a probability distribution over outcomes by treating each observed trade as a training instance.

In Section~\ref{sec:connections}, we go on to show that the class of \emph{convex} cost function based markets exactly corresponds to the class of Follow the Regularized Leader learning algorithms~\cite{SS07,HK08,H09} in which weights are chosen at each time step to minimize a combination of empirical loss and a convex regularization term.  This allows us to interpret the selection of a cost function for the market as the selection of a regularizer for the learning problem.  Furthermore, we prove an equivalence between another common class of prediction markets, \emph{market scoring rules}, and convex cost function based markets,\footnote{A similar but weaker correspondence between market scoring rules and cost function based markets was discussed in \citet{CP07} and \citet{ADPWY09}.} which immediately implies that market scoring rules can be interpreted as Follow the Regularized Leader algorithms too.  These connections provide insight into why it is that prediction markets tend to yield such accurate estimates in practice.

Before describing our results in more detail, we review the relevant concepts and results from the literature on prediction markets and no-regret learning in Sections~\ref{sec:predmarkets} and~\ref{sec:experts}.

\section{Prediction Markets}
\label{sec:predmarkets}

In recent years, a variety of compelling prediction market mechanisms have been proposed and studied, including standard call market mechanisms and Pennock's dynamic parimutuel markets~\cite{P04}.  In this work we focus on two broad classes of mechanisms: Hanson's market scoring rules~\cite{H03,H07} and cost function based prediction markets as described in \citet{CP07}.  We also briefly discuss the related class of Sequential Convex Parimutuel Mechanisms~\cite{ADPWY09} in Section~\ref{sec:scpm}.

\subsection{Market Scoring Rules}

\emph{Scoring rules} have long been used in the evaluation of probabilistic forecasts.  In the context of prediction markets and elicitation, scoring rules are used to encourage individuals to make careful assessments and truthfully report their beliefs~\cite{Savage:71,GKO05,LPS08}.  In the context of machine learning, scoring rules are used as loss functions to evaluate and compare the performance of different algorithms~\cite{BSS05,RW09}.

Formally, let $\{1,\cdots,N\}$ be a set of mutually exclusive and exhaustive outcomes of a future event.  A scoring rule $\vec{s}$ maps a probability distribution $\p$ to a score $s_i(\p)$ for each outcome $i$, with $s_i(\p)$ taking values in the extended real line $[-\infty, \infty]$.  Intuitively, this score represents the reward of a forecaster might receive for predicting the distribution $\p$ if the outcome turns out to be $i$.  A scoring rule is said to be \emph{regular} relative to the probability simplex $\simp_N$ if $\sum_{i=1}^N p_i s_i(\p\,') \in [-\infty, \infty)$ for all $\p, \p\,' \in \simp_N$, with $\sum_{i=1}^N p_i s_i(\p) \in (-\infty, \infty)$.  This implies that $s_i(\vec{p})$ is finite whenever $p_i > 0$.  A scoring rule is said to be \emph{proper} if a risk-neutral forecaster who believes the true distribution over outcomes to be $\p$ has no incentive to report any alternate distribution $\p\,'$, that is, if $\sum_{i=1}^N p_i s_i(\p) \geq \sum_{i=1}^N p_i s_i(\p\,')$ for all distributions $\p\,'$.   The rule is \emph{strictly proper} if this inequality holds with equality only when $\p = \p\,'$.  

Two examples of regular, strictly proper scoring rules commonly used in both elicitation and in machine learning are the the quadratic scoring rule~\cite{B50}:
\begin{equation}
s_i(\p) = a_i + b\left(2p_i - \sum_{i=1}^N p_i^2 \right)
\label{eqn:qsr}
\end{equation}
and the logarithmic scoring rule~\cite{G52}:
\begin{equation}
s_i(\p) = a_i + b \log(p_i) 
\label{eqn:lsr}
\end{equation}
with arbitrary parameters $a_1, \cdots, a_N$ and parameter $b > 0$.
The uses and properties of scoring rules are too extensive to cover in detail here.  For a nice survey, see \citet{Gneiting:07}.

\emph{Market scoring rules} were developed by \citet{H03,H07} as a method of using scoring rules to pool opinions from many different forecasters.  Market scoring rules are sequentially shared scoring rules.  Formally, the market maintains a current probability distribution $\p$.  At any time, a trader can enter the market and change  this distribution to an arbitrary distribution $\p\,'$ of her choice.\footnote{While $\p\,'$ may be arbitrary, in some market scoring rules, such as the LMSR, distributions that place a weight of 0 on any outcome are not allowed because it requires the trader to pay infinite amount of money if the outcome with reported probability 0 actually happens.}  If the outcome turns out to be $i$, she receives a (possibly negative) payoff of $s_i(\p\,') - s_i(\p)$.  For example, in the popular Logarithmic Market Scoring Rule (LMSR), which is based on the logarithmic scoring rule in Equation~\ref{eqn:lsr}, a trader who changes the distribution from $\p$ to $\p\,'$ receives a payoff of $b \log (p'_i / p_i)$.

Since the trader has no control over $\p$, a myopic trader who believes the true distribution to be $\vec{r}$ maximizes her expected payoff by maximizing $\sum_i r_i s_i (\p\,')$.  Thus if $\vec{s}$ is a strictly proper scoring rule, traders have an incentive to change the market's distribution to match their true beliefs.  The idea is that if traders update their own beliefs over time based on market activity, the market's distribution should eventually converge to the collective beliefs of the population.

Each trader in a market scoring rule is essentially responsible for paying the previous trader's score.  Thus the market maker is responsible only for paying the score of the final trader.  Let $\p_0$ be the initial probability distribution of the market.  The worst case loss of the market maker is then
\[
\max_{i \in \{1,\cdots,N\}} \max_{\p \in \simp_N} \left(s_i(\p) - s_i(\p_0) \right) .
\]
The worst case loss of the market maker running an LMSR initialized to the uniform distribution is $b \log N$.

Note that the parameters $a_1, \cdots, a_N$ of the logarithmic scoring rule do not affect either the payoff of traders or the loss of the market maker in the LMSR.  For simplicity, in the remainder of this paper when discussing the LMSR we assume that $a_i = 0$ for all $i$.

\subsection{Cost Function Based Markets}

As before, let $\{1,\cdots,N\}$ be a set of mutually exclusive and exhaustive outcomes of an event.  In a cost function based market, a market maker offers a security corresponding to each outcome $i$. The security associated with outcome $i$ pays off \$1 if $i$ happens, and \$0 otherwise.\footnote{The dynamic parimutuel market falls outside this framework since the winning payoff depends on future trades.}

Different mechanisms can be used to determine how these securities are priced. Each mechanism is specified using a differentiable \emph{cost function} $C: \real^N \rightarrow \real$. This cost function is simply a potential function describing the amount of money currently wagered in the market as a function of the quantity of shares purchased.  If $q_i$ is the number of shares of security $i$ currently held by traders, and a trader would like to purchase $r_i$ shares of each security (where $r_i$ could be zero or even negative, representing the sale of shares), the trader must pay $C(\q + \r) - C(\q)$ to the market maker.  The instantaneous price of security $i$ (that is, the price per share of an infinitely small number of shares) is then
$p_i = \partial C(\q)/\partial q_i$.

We say that a cost function is \emph{valid} if the associated prices satisfy two simple conditions:
\begin{enumerate}
\item For every $i \in \{1,\cdots,N\}$ and every $\vec{q} \in \real^N$, $p_i(\vec{q}) \geq 0$.
\item For every $\vec{q} \in \real^N$, $\sum_{i=1}^N p_i(\vec{q}) =1$ ~.
\end{enumerate}
The first condition ensures that the price of a security is never negative.  If the current price of the security associated with an outcome $i$ were negative, a trader could purchase shares of this security at a guaranteed profit.  The second condition ensures that the prices of all securities sum to 1.  If the prices summed to something less than (respectively, greater than) 1, then a trader could purchase (respectively, sell) small equal quantities of each security for a guaranteed profit.  Together, these conditions ensure that there are no arbitrage opportunities in the market.

These conditions also ensure that the current prices can always be viewed as a valid probability distribution over the outcome space.  In fact, these prices represent the market's current estimate of the probability that outcome $i$ will occur.

The following theorem gives sufficient and necessary conditions for the cost function $C$ to be valid.  While these properties of cost functions have been discussed elsewhere~\cite{CP07,ADPWY09}, the fact that they are both sufficient and necessary for any valid cost function $C$ is important for our later analysis.  As such, we state the full proof here for completeness.

\begin{theorem}
A cost function $C$ is valid if and only if it satisfies the following three properties:
\begin{enumerate}
\item {\sc Differentiability:} The partial derivatives $\partial C(\q)/\partial q_i$ exist for all $\vec{q}\in \real^N$ and $i\in\{1, \dots, N\}$.
\item {\sc Increasing Monotonicity:} For any $\vec{q}$ and $\vec{q}\,'$, if $\vec{q} \geq \vec{q}\,'$, then $C(\vec{q}) \geq C(\vec{q}\,')$. 
\item {\sc Positive Translation Invariance:} For any $\vec{q}$ and any constant $k$, $C(\vec{q} + k\vec{1}) = C(\vec{q}) + k$.
\end{enumerate}
\label{thm:validcostfunc}
\end{theorem}
\begin{proof}
Differentiability is necessary and sufficient for the price functions to be well-defined at all points.  It is easy to see that requiring the cost function to be monotonic is equivalent to requiring that $p_i(\vec{q}) \geq 0$ for all $i$ and $\vec{q}$.  We will show that requiring positive translation invariance is equivalent to requiring that the prices always sum to one.

First, assume that $\sum_{i=1}^N p_i(\vec{q}) = 1$ for all $\vec{q}$.  For any fixed value of $\vec{q}$, define $\vec{u} = \vec{u}(a) = \q + a \vec{1}$ and let $u_i$ be the $i$th component of $\vec{u}$.  Then for any $k$,
\begin{eqnarray*}
C(\vec{q} + k \vec{1}) - C(\vec{q})
&=& \int_{0}^k \frac{d C(\vec{q} + a \vec{1})}{d a} da
\\
&=& \int_{0}^k \sum_{i=1}^N \frac{\partial C(\vec{u})}{\partial u_i} \frac{\partial u_i}{\partial a} da 
\\
&=&  \int_{0}^k \sum_{i=1}^N p_i(\vec{u}) da 
= k~.
\end{eqnarray*}
This is precisely translation invariance.

Now assume instead that positive translation invariance holds.  Fix any arbitrary $\vec{q}\,'$ and $k$ and define $\vec{q} = \vec{q}\,' + k\vec{1}$.  Notice that by setting $\vec{q}\,'$ and $k$ appropriately, we can make $\vec{q}$ take on any arbitrary values. We have,
\[
\frac{\partial C(\vec{q})}{\partial k}
= \sum_{i=1}^N\frac{\partial C(\vec{q})}{\partial q_i}  \frac{\partial q_i}{\partial k}
= \sum_{i=1}^N p_i(\vec{q}).
\]
By translation invariance, $C(\vec{q}~'+k\vec{1}) = C(\vec{q}~') + k$. Thus, 
\[
\frac{\partial C(\vec{q})}{\partial k} = \frac{\partial (C(\vec{q}~') + k) }{\partial k}=1. 
\]
Combining the two equations, we have $\sum_{i=1}^N p_i(\vec{q}) = 1$.
\end{proof}

One quantity that is useful for comparing different market mechanisms is the worst-case loss of the market maker, 
\[
\max_{\q \in \real^N} \left( \max_{i \in \{1,\cdots,N\}} q_i -  (C(\q) - C(\vec{0}))  \right) ~.
\]
This is simply the difference between the maximum amount that the market maker might have to pay the winners and the amount of money collected by the market maker.

The Logarithmic Market Scoring Rule described above can be specified as a cost function based prediction market~\cite{H03,CP07}.  Then cost function of the LMSR is
\[
C(\vec{q}) = b \log \sum_{i=1}^N \e^{q_{i}/b} ~,
\]
and the corresponding prices are
\[
p_{i}(\vec{q})
= \frac{\partial C(\vec{q})}{\partial q_{i}}
= \frac{ \e^{q_{i}/b}}{\sum_{j=1}^N \e^{q_{j}/b}} ~.
\]
This formulation is equivalent to the market scoring rule formulation in the sense that a trader who changes the market probabilities from $\r$ to $\r\,'$ in the MSR formulation receives the same payoff for every outcome $i$ as a trader who changes the quantity vectors from any $\q$ to $\q\,'$ such that $p(\q) = \r$ and $p(\q\,') = \r\,'$ in the cost function formulation.

\section{Learning from Expert Advice}
\label{sec:experts}

We now briefly review the problem of learning from expert advice.  In this framework, an algorithm makes a sequence of predictions based on the advice of a set of $N$ \emph{experts} and receives a corresponding sequence of \emph{losses}.\footnote{This framework could be formalized equally well in terms of \emph{gains}, but losses are more common in the literature.}  The goal of the algorithm is to achieve a \emph{cumulative loss} that is ``almost as low'' as the cumulative loss of the best performing expert in hindsight.  No statistical assumptions are made about these losses. Indeed, algorithms are expected to perform well even if the sequence of losses is chosen by an adversary.

Formally, at every time step $t \in \{1,\cdots,T\}$, every expert $i \in \{1,\cdots, N\}$ receives a loss $\loss_{i,t}\in [0,1]$.  The cumulative loss of expert $i$ at time $T$ is then defined as $\Loss_{i,T} = \sum_{t=1}^T \loss_{i,t}$.  An algorithm $\alg$ maintains a weight $w_{i,t}$ for each expert $i$ at time $t$, where $\sum_{i=1}^n w_{i,t} = 1$. These weights can be viewed as a distribution over the experts.  The algorithm then receives its own instantaneous loss $\loss_{\alg,t} =\sum_{i=1}^n w_{i,t} \loss_{i,t}$, which can be interpreted as the expected loss the algorithm would receive if it always chose an expert to follow according to the current distribution. The cumulative loss of $\alg$ up to time $T$ is defined in the natural way as $\Loss_{\alg,T} = \sum_{t=1}^T \loss_{\alg,t}=\sum_{t=1}^T \sum_{i=1}^n w_{i,t} \loss_{i,t}$.

It is unreasonable to expect the algorithm to achieve a small cumulative loss if none of the experts perform well.  As such, it is typical to measure the performance of an algorithm in terms of its \emph{regret}, defined to be the difference between the cumulative loss of the algorithm and the loss of the best performing expert, that is,
\[
\Loss_{\alg,T} - \min_{i \in \{1,\cdots,N\}} \Loss_{i,T} .
\]
An algorithm is said to have \emph{no regret} if the average per time step regret approaches $0$ as $T$ approaches infinity.

The popular Randomized Weighted Majority (WM) algorithm~\cite{LW94,FS97} is an example of a no-regret algorithm.  Weighted Majority uses weights
\[
w_{i,t} = \frac{\e^{-\eta \Loss_{i,t}}}{\sum_{j=1}^n \e^{-\eta\Loss_{j,t}}} ,
\]
where $\eta > 0$ is a tunable parameter known as the \emph{learning rate}.  It is well known that the regret of WM after $T$ trials can be bounded as
\[
\Loss_{WM(\eta),T} - \min_{i \in \{1,\cdots,N\}} \Loss_{i,T} \leq \eta T + \frac{\log N}{\eta}.
\]
When $T$ is known in advance, setting $\eta=\sqrt{\log N /T}$ yields the standard $O(\sqrt{T \log N})$ regret bound.

It has been shown that the weights chosen by Weighted Majority are precisely those that minimize a combination of empirical loss and an entropic regularization term~\cite{KW97,KW99,HW09}.  More specifically, the weights at time $t$ are precisely those that  minimize
\[
\sum_{i=1}^N w_{i} \Loss_{i,t-1} - \frac{1}{\eta} \entro(\vec{w}) 
\]
among all $\vec{w} \in \simp_N$, where $\entro$ is the entropy.  This makes Weighted Majority an example of broader class of algorithms collectively known as \emph{Follow the Regularized Leader} algorithms~\cite{SS07,HK08,H09}.  This class of algorithms grew out of the following fundamental insight of \citet{KV05}.

Consider first the aptly named \emph{Follow the Leader} algorithm, which chooses weights at time $t$ to minimize $\sum_{i=1}^N w_{i,t} \Loss_{i,t-1}$.  This algorithm simply places all of its weight on the single expert (or set of experts) with the best performance on previous examples.  As such, this algorithm can be highly unstable, dramatically changing its weights from one time step to the next.  It is easy to see that Follow the Leader suffers $\Omega(T)$ regret in the worst case when the best expert changes frequently. For example, if there are only two experts with losses starting at $\angles{1/2,0}$ and then alternating $\angles{0,1}, \angles{1,0}, \angles{0,1}, \angles{1,0}, \cdots$, then FTL places a weight of 1 on the losing expert at every point in time.

To overcome this instability, \citet{KV05} suggested adding a random perturbation to the empirical loss of each expert, and choosing the expert that minimizes this perturbed loss.\footnote{A very similar algorithm was originally developed and analyzed by Hannan in the 1950s~\cite{H57}.}  However, in general this perturbation need not be random.  Instead of adding a random perturbation, it is possible to gain the necessary stability by adding a \emph{regularizer} $\Reg$ and choosing weights to minimize
\begin{equation}
\sum_{i=1}^N w_{i,t} \Loss_{i,t-1} + \frac{1}{\eta} \Reg(\vec{w}_t) ~.
\label{eqn:ftrl}
\end{equation}
This Follow the Regularized Leader (FTRL) approach gets around the instability of FTL and guarantees low regret for a wide variety of regularizers, as evidenced by the following bound of \citet{HK08}.
\begin{lemma}[\citet{HK08}] For any regularizer $\Reg$, the regret of FTRL can be bounded as
\begin{eqnarray*}
\lefteqn{\Loss_{FTRL(\Reg,\eta),T} - \min_{i \in \{1,\cdots,N\}} \Loss_{i,T} }
\\
&\leq& 
\sum_{t=1}^T \sum_{i=1}^N \loss_{i,t} (w_{i,t} - w_{i,t+1})
+ \frac{1}{\eta} \left(\Reg(\vec{w}_{T}) - \Reg(\vec{w}_0)\right) ~.
\end{eqnarray*}
\label{lem:ftrlbound}
\end{lemma}

This lemma quantifies the trade-off that must be considered when choosing a regularizer.  If the range of the regularizer is too small, the weights will change dramatically from one round to the next, and the first term in the bound will be large.  On the other hand, if the range of the regularizer is too big, the weights that are chosen will be too far from the true loss minimizers and the second term will blow up.

It is generally assumed that the regularizer $\Reg$ is strictly convex.  This assumption ensures that Equation~\ref{eqn:ftrl} has a unique minimizer and that this minimizer can be computed efficiently.  \citet{H09} shows that if $\Reg$ is strictly convex then it is possible to achieve a regret of $O(\sqrt{T})$.  In particular, by optimizing $\eta$ appropriately the regret bound in Lemma~\ref{lem:ftrlbound} can be upper bounded by
\begin{equation}
2 \sqrt{2 \lambda \max_{\vec{w},\vec{w}\,' \in \simp_N} (\Reg(\vec{w}) - \Reg(\vec{w}\,'))  T}
\label{eqn:lambdabound}
\end{equation}
where $\lambda = \max_{\loss \in [0,1]^N, \vec{w} \in \simp_N} \loss^T [\nabla^2 \Reg(\vec{w})]^{-1} \loss$.

\section{Interpreting Prediction Markets as No-Regret Learners}
\label{sec:connection}

With this foundation in place, we are ready to describe how any bounded loss market maker can be interpreted as an algorithm for learning from expert advice.  The key idea is to equate the \emph{trades} made in the market with the \emph{losses} observed by the learning algorithm.  We can then view the market maker as essentially learning a probability distribution over outcomes by treating each observed trade as a training instance.  

More formally, consider any cost function based market maker with instantaneous price functions $p_i$ for each outcome $i$.  We convert such a market maker to an algorithm for learning from expert advice by setting the weight of expert $i$ at time $t$ using
\begin{equation}
w_{i,t} = p_i(-\epsilon \vec{\Loss}_{t-1}) ,
\label{eqn:weights}
\end{equation}
where $\epsilon > 0$ is a tunable parameter and $\vec{\Loss}_{t-1} = \langle \Loss_{1,t-1},\cdots,\Loss_{N,t-1} \rangle$ is the vector of cumulative losses at time $t-1$.  In other words, the weight on expert $i$ at time $t$ in the learning algorithm is the instantaneous price of security $i$ in the market when $-\epsilon \Loss_{j,t-1}$ shares have been purchased (or $\epsilon \Loss_{j,t-1}$ shares have been sold) of each security $j$.  We discuss the role of the parameter $\epsilon$ in more detail below.

First note that for any valid cost function based prediction market, setting the weights as in Equation~\ref{eqn:weights} entails valid expert learning algorithm.  Since the prices of any valid prediction market must be non-negative and sum to one, the weights of the resulting algorithm are guaranteed to satisfy these properties too.  Furthermore, the weights are a function of only the past losses of each expert, which the algorithm is permitted to observe.

Below we show that applying this conversion to any bounded-loss market maker with slowly changing prices yields a learning algorithm with $O(\sqrt{T})$ regret.  The quality of the regret bound obtained depends on the trade-off between market maker loss and how quickly the prices change.  We then show how this bound can be used to rederive the standard regret bound of Weighted Majority, the converse of the result of \citet{CFLPW08}.

\subsection{A Bound on Regret}
\label{sec:regretbound}

In order to derive a regret bound for the learning algorithm defined in Equation~\ref{eqn:weights}, it is necessary to make some restrictions on how quickly the prices in the market change.  If market prices change too quickly, the resulting learning algorithm will be unstable and will suffer high worst-case regret, as was the case with the naive Follow The Leader algorithm described in Section~\ref{sec:experts}.  To capture this idea, we introduce the notion of $\phi$-stability, defined as follows.
\begin{definition}
We say that a set of price functions $\vec{p}$ is \emph{$\phi$-stable} for a constant $\phi$ if $p_i$ is continuous and piecewise differentiable for all $i \in \{1,\cdots,N\}$ and $\sum_{i=1}^N \sum_{j=1}^N \abs{D_{i,j}(\vec{t})} \leq \phi$ for all $\vec{t}$, where
\[
D_{i,j}(\vec{t}) = 
\begin{cases}
\left. \frac{\partial p_i(\q)}{\partial q_j} \right|_{\q = \vec{t}} 
& \textrm{if $\frac{\partial p_i(\q)}{\partial q_j}$ is defined at $\vec{t}$,}
\\
0  
& \textrm{otherwise.}
\end{cases}
\]
\end{definition}

Defining $\phi$-stability in terms of the $D_{i,j}$ allows us to quantify how slowly the prices change even when the price functions are not differentiable at all points.  We can then derive a regret bound for the resulting learning algorithm using the following simple lemma.  This lemma states that when the quantity vector in the market is $\q$, if the price functions are $\phi$-stable, then the amount of money that the market maker would collect for the purchase of a small quantity $r_i$ of each security $i$ is not too far from the amount that the market maker would have collected had he instead priced the shares according to the fixed price $\vec{p}(\q)$.

\begin{lemma}
Let $C$ be any valid cost function yielding $\phi$-stable prices. For any $\epsilon > 0$, any $\q \in \real^N$, and any $\r \in \real^N$ such that $\abs{r_i} \leq \epsilon$ for $i \in
\{1,\cdots,N\}$,
\[
\abs{
\left(C(\q+\r) - C(\q) \right)
- \sum_{i=1}^N p_i(\q) r_i
}
\leq \frac{\epsilon^2 \phi}{2} ~.
\]
\label{lem:pricingdiffbound}
\label{LEM:PRICINGDIFFBOUND} 
\end{lemma}
The proof is in Appendix~\ref{app:pricingdiffbound}.

With this lemma in place, we are ready to derive the regret bound.  In the following theorem, it is assumed that $T$ is known a priori and therefore can be used to set $\epsilon$.  If $T$ is not known in advance, a standard ``doubling trick'' can be applied~\cite{C-B+97}.  The idea behind the doubling trick is to partition time into periods of exponentially increasing length, restarting the algorithm each period.  This leads to similar bounds with only an extra factor of $\log(T)$.

\begin{theorem}
  Let $C$ be any valid cost function yielding $\phi$-stable prices.  Let $B$ be a bound on the worst-case loss of the market maker mechanism associated with $C$.  Let $\alg$ be the expert learning algorithm with weights as in Equation~\ref{eqn:weights} with $\epsilon = \sqrt{2B/(\phi T)}$.  Then for any sequence of expert losses $\loss_{i,t} \in [0,1]$ over $T$ time steps, \[ \Loss_{\alg,T} - \min_{i \in \{1,\cdots,N\}} \Loss_{i,T} \leq \sqrt{2 B \phi T} ~.  \] 
\label{thm:mainreduction}
 \end{theorem} 
\begin{proof}
By setting the weights as in Equation~\ref{eqn:weights}, we are essentially simulating a market over $N$ outcomes.  Let $r_{i,t}$ denote the number of shares of outcome $i$ purchased at time step $t$ in this simulated market, and denote by $\r_t$ the vector of these quantities for all $i$.  Note that $r_{i,t}$ is completely in our control since we are simply simulating a market, thus we can choose to set $r_{i,t} = -\epsilon \loss_{i,t}$ for all $i$ and $t$.  We have that $r_{i,t} \in [-\epsilon,0]$ for all $i$ and $t$ since $\loss_{i,t} \in [0,1]$.  Let $q_{i,t} = \sum_{t'=1}^t r_{i,t'}$ be the total number of outstanding shares of security $i$ after time $t$, with $\q_t$ denoting the vector over all $i$.  The weight assigned to expert $i$ at round $t$ of the learning algorithm corresponds to the instantaneous price of security $i$ in the simulated market immediately before round $t$, that is,  $w_{i,t} = p_i(-\epsilon \vec{\Loss}_{t-1}) = p_i(\q_{t-1})$.

By the definition of worst-case market maker loss, $\max_{i} q_{i,t}  - (C(\q_t) - C(\vec{0})) \leq B$.  It is easy to see that we can rewrite the left-hand side of this equation to obtain
\[
\max_{i \in \{1,\cdots,N\}} \sum_{t=1}^T r_{i,t} 
- \sum_{t=1}^t  \left(C(\q_{t}) - C(\q_{t-1}) \right)
\leq B ~.
\]
From Lemma~\ref{lem:pricingdiffbound}, this gives us that
\[
\max_{i \in \{1,\cdots,N\}} \sum_{t=1}^T r_{i,t} 
- \sum_{t=1}^t  
\left( \sum_{i=1}^N  p_i(\q_{t-1}) r_{i,t}
+ \frac{\epsilon^2 \phi}{2} \right)
\leq B .
\]
Substituting $p_i(\q_{t-1}) = w_{i,t}$ and $r_{i,t} = - \epsilon \loss_{i,t}$, we get
\[
\max_{i \in \{1,\cdots,N\}} \sum_{t=1}^T \left(- \epsilon \loss_{i,t} \right)
- \sum_{t=1}^t  \sum_{i=1}^N w_{i,t} \left( - \epsilon \loss_{i,t} \right)
\leq B + \frac{\epsilon^2 \phi T}{2} 
\]
and so
\begin{eqnarray*}
\Loss_{\alg,T} - \min_{i \in \{1,\cdots,N\}} \Loss_{i,T}
&=& \sum_{t=1}^t  \sum_{i=1}^N w_{i,t} \loss_{i,t}
- \min_i \sum_{t=1}^T \loss_{i,t}
\\
&\leq& \frac{B}{\epsilon} + \frac{\epsilon \phi T}{2}  .
\end{eqnarray*}
Setting $\epsilon = \sqrt{2B/(\phi T)}$ yields the bound.
\end{proof}

\subsection{Rederiving the Weighted Majority Bound}
\label{sec:wmbound}

\citet{CFLPW08} showed that the Weighted Majority regret bound can be used as a starting point to rederive the worst case loss of $b \log N$ of an LMSR market maker.  Here we show that the converse is also true; by applying Theorem~\ref{thm:mainreduction}, we can rederive the Weighted Majority bound from the bounded market maker loss of LMSR.

In order to apply Theorem~\ref{thm:mainreduction}, we must provide a bound on how quickly LMSR prices can change.  This is given in the following lemma, the proof of which is in Appendix~\ref{app:lmsrderiv}.

\begin{lemma}
Let $\p$ be the pricing function of a LMSR with parameter $b > 0$.  Then
\[
\sum_{i=1}^N \sum_{j=1}^N \abs{\frac{\partial p_i(\q)}{\partial q_j}}
\leq \frac{2}{b} ~.
\]
\label{lem:lmsrderiv}
\label{LEM:LMSRDERIV} 
\end{lemma}

Using Equation~\ref{eqn:weights} to transform the LMSR into a learning algorithm, we end up with weights
\[
w_{i,t}
= \frac{ \e^{- \epsilon\Loss_{i,t-1}/b}}{\sum_{j=1}^N \e^{- \epsilon \Loss_{j,t-1}/b}}  ~.
\]
Setting $\epsilon = \sqrt{2B/(\phi T)} = b \sqrt{\log N/T}$, we see that these weights are equivalent to those used by Weighted Majority with the learning rate $\eta = \epsilon / b = \sqrt{\log N/T}$.  As mentioned above, this is the optimal setting of $\eta$.  Notice that these weights do not depend on the value of the parameter $b$ in the prediction market.

We can now apply Theorem~\ref{thm:mainreduction} to rederive the standard Weighted Majority regret bound stated in Section~\ref{sec:experts}.  In particular, setting $B = b \log N$ and $\phi = 2/b$, we get that when $\eta = \sqrt{\log(N)/T}$,
\[
\Loss_{WM,T} - \min_{i \in \{1,\cdots,N\}} \Loss_{i,T}
\leq 2 \sqrt{T \log N} ~.
\]

\section{Connections Between Market Scoring Rules, Cost Functions, and Regularization}
\label{sec:connections}

In this section, we establish the formal connections among market scoring rules, cost function based markets, and the class of Follow the Regularized Leaders algorithms. We start with a representation theorem for cost function based markets, which is crucial in our later analysis. 

\subsection{A Representation Theorem for Convex Cost Functions}
\label{sec:riskmeasures}

In this section we show a representation theorem for convex cost functions.  The proof of this theorem relies on the connection between convex cost functions and a class of functions known in the finance literature as convex risk measures, which was first noted by \citet{ADPWY09}.  Convex risk measures were originally introduced by \citet{FS02} to model different attitudes towards risk in financial markets.  A \emph{risk measure} $\risk$ can be viewed as a mapping from a vector of returns (corresponding to each possible outcome of an event) to a real number.  The interpretation is that a vector of returns $\vec{x}$ is ``preferred to'' the vector $\vec{x}\,'$ under a risk measure $\risk$ if and only if $\risk(\vec{x}) < \risk(\vec{x}\,')$.

Formally, a function $\risk$ is a \emph{convex risk measure} if it satisfies the following three properties:
\begin{enumerate}
\item {\sc Convexity:} $\risk(\vec{x})$ is a convex function of $\vec{x}$. 

\item {\sc Decreasing Monotonicity:} For any $\vec{x}$ and $\vec{x}\,'$, if $\vec{x} \geq \vec{x}\,'$, then $\risk(\vec{x}) \leq \risk(\vec{x}\,')$.

\item {\sc Negative Translation Invariance:} For any $\vec{x}$ and value $k$, $\risk(\vec{x} + k\vec{1}) = \risk(\vec{x}) - k$.
\end{enumerate}

The financial interpretations of these properties are not important in our setting.  More interesting for us is that \citet{FS02} provide a representation theorem that states that a function $\risk$ is a convex risk measure if and only if it can be represented as
\[
\risk(\vec{x}) = \sup_{\vec{p}\in \simp_N} \left( - \sum_{i=1}^N p_i x_i - \alpha (\vec{p})\right)
\]
where $\alpha:\simp_N \to (-\infty, \infty]$ is a convex, lower semi-continuous function referred to as a \emph{penalty function}.  
This fact is useful because it allows us to obtain the following result, which was alluded to informally by \citet{ADPWY09}.  The full proof is included here for completeness.

\begin{lemma}
A function $C$ is a valid convex cost function if and only if it is differentiable and can be represented as 
\begin{equation}
C(\vec{q}) = \sup_{\vec{p}\in \simp_N} \left(\sum_{i=1}^N p_i q_i - \alpha (\vec{p})\right)
\label{eqn:riskrep}
\end{equation}
for a convex and lower semi-continuous function $\alpha$.  Furthermore, for any quantity vector $\q$, the price vector $\vec{p}(\q)$ corresponding to $C$ is the distribution $\vec{p}$ maximizing $\sum_{i=1}^N p_i q_i - \alpha (\vec{p})$.
\label{lem:costrepresentation}
\end{lemma}
\begin{proof}
Consider any differentiable function $C : \real^N \rightarrow \real$.  Let $\risk(\q) = C(-\q)$.  Clearly by definition, $\risk$ satisfies decreasing monotonicity if and only if $C$ satisfies increasing monotonicity, and $\risk$ satisfies negative translation invariance if and only if $C$ satisfies positive translation invariance.  Furthermore, $\risk$ is convex if and only if $C$ is convex.  By Theorem~\ref{thm:validcostfunc}, this implies that $C$ is a valid convex cost function if and only if $\risk$ is a convex risk measure.  The first half of the lemma then follows immediately from the representation theorem of \citet{FS02}.

Now, because $\alpha (\vec{p})$ is guaranteed to be convex, $\sum_{i=1}^N p_i q_i - \alpha (\vec{p})$ is a concave function of $\vec{p}$.  The constraints $\sum_{i=1}^N p_i = 1$ and $p_i \geq 0$ define a closed convex feasible set. Thus, the problem of maximizing $\sum_{i=1}^N p_i q_i - \alpha (\vec{p})$ with respect to $\vec{p}$ has a global optimal solution and first-order KKT conditions are both necessary and sufficient. Let $\vec{p}\,^*(\vec{q})$ denote an optimal $\vec{p}$ for this optimization problem. Then, $C(\vec{q}) = \sum_{i=1}^N p_i^*(\q) q_i - \alpha(\vec{p}\,^*(\vec{q}))$.  By the envelope theorem~\cite{MS02}, if $C(\vec{q})$ is differentiable, we have that for any $i$, $p_i^*(\vec{q}) = \partial C(\vec{q})/\partial q_i =p_i(\vec{q})$.  Thus the market prices are precisely those which maximize the inner expression of the cost function.  
\end{proof}

Furthermore, by a version of the envelope theorem~\cite{K93}, to ensure that $C$ is differentiable, it is sufficient to show that $\alpha$ is strictly convex and differentiable.

\begin{corollary}
A function $C$ is a valid convex cost function if it can be represented as in Equation~\ref{eqn:riskrep} for a strictly convex and differentiable function $\alpha$.  For any $\q$, the price vector $\vec{p}(\q)$ is the distribution $\vec{p}$ maximizing $\sum_{i=1}^N p_i q_i - \alpha (\vec{p})$.
\label{cor:costrepresentation}
\end{corollary}

The ability to represent any valid cost function in this form allows us to define a bound on the worst-case loss of the market maker in terms of the penalty function of the corresponding convex risk measure.

\begin{lemma}
The worst-case loss of the market maker defined by the cost function in Equation~\ref{eqn:riskrep} is no more than
\[
\sup_{\vec{p},\vec{p}\,' \in \simp_N} \left(\alpha(\vec{p}) - \alpha(\vec{p}\,')\right).
\]
\label{lem:worstcaseloss}
\end{lemma}
\begin{proof}
The worst-case loss of the market maker is
\begin{eqnarray*}
\lefteqn{\max_{\q \in \real^N} \left( \max_{i\in \{1,\cdots,N\}} q_i -  C(\q)\right) + C(\vec{0}) }
\\
&=& \max_{\q \in \real^N}
\left( \max_{i \in \{1,\cdots,N\}} q_i
- \sup_{\vec{p}\in \simp_N} \left(\sum_{i=1}^N p_i q_i - \alpha (\vec{p})\right) \right)
\\
&& + \sup_{\vec{p}\,'\in \simp_N} \left(-\alpha (\vec{p}\,')\right)
\\
&\leq& \max_{\q \in \real^N}
\!\left(\!\max_{i \in \{1,\cdots,N\}} q_i
- \!\left(\!\sup_{\vec{p}\in \simp_N} \sum_{i=1}^N p_i q_i 
- \sup_{\vec{p}\in \simp_N} \left(\alpha (\vec{p})\right)\!\right)\!\right)
\\
&& + \sup_{\vec{p}\,'\in \simp_N} \left(-\alpha (\vec{p}\,')\right)
\\
&=&
\max_{\q \in \real^N} \left( \max_{i \in \{1,\cdots,N\}} q_i -\max_{i \in \{1,\cdots,N\}} q_i \right)
+ \sup_{\vec{p}\in \simp_N} \left(\alpha (\vec{p})\right) 
\\
&& + \sup_{\vec{p}\,'\in \simp_N} \left(-\alpha (\vec{p}\,')\right)
\\
&=& \sup_{\vec{p},\vec{p}\,' \in \simp_N} \left(\alpha(\vec{p}) - \alpha(\vec{p}\,')\right).
\end{eqnarray*}
The inequality follows from the fact that for any functions $f$ and $g$ over any domain $\mathcal{X}$, $\sup_{x \in \mathcal{X}} (f(x) - g(x)) \geq \sup_{x \in \mathcal{X}} f(x) - \sup_{x' \in \mathcal{X}} g(x').$
\end{proof}

\subsection{Convex Cost Functions and Market Scoring Rules}

As described in Section~\ref{sec:predmarkets}, the Logarithmic Market Scoring Rule market maker can be defined as either a market scoring rule or a cost function based market.  The LMSR is not unique in this regard.  As we show in this section, any regular, strictly proper market scoring rule with differentiable scoring functions can be represented as a cost function based market.  Likewise, any convex cost function satisfying a few mild conditions corresponds to a market scoring rule.  As long as the market probabilities are nonzero, the market scoring rule and corresponding cost function based market are equivalent.  More precisely, a trader who changes the market probabilities from $\r$ to $\r\,'$ in the market scoring rule is guaranteed to receive the same payoff for every outcome $i$ as a trader who changes the quantity vectors from any $\q$ to $\q\,'$ such that $p(\q) = \r$ and $p(\q\,') = \r\,'$ in the cost function formulation as long as every component of $\r$ and $\r\,'$ is nonzero.  Moreover, any price vector that is achievable in the market scoring rule (that is, any $\p$ for which $s_i(\vec{p})$ is finite for all $i$) is achievable by the cost function based market.

The fact that there exists a correspondence between certain market scoring rules and certain  cost function based markets was noted by \citet{CP07}.  They pointed out that the MSR with scoring function $\vec{s}$ and the cost function based market with cost function $C$ are equivalent if for all $\q$ and all outcomes $i$, $C(\q) = q_i - s_i(\vec{p})$.  However, they did not provide any guarantees about the circumstances under which this condition can be satisfied.  \citet{ADPWY09} also made use of the equivalence between markets when this strong condition holds.  Our result gives very general precise conditions under which an MSR is equivalent to a cost function based market.

Recall from Lemma~\ref{lem:costrepresentation} that any convex cost function $C$ can be represented as $C(\q) = \sup_{\vec{p}\in \simp_N} \left(\sum_{i=1}^N p_i q_i - \alpha (\vec{p})\right)$ for a convex function $\alpha$.  Let $\alpha_C$ denote the function $\alpha$ corresponding to the cost function $C$.  In the following, we consider cost functions derived from scoring rules $\vec{s}$ by setting
\begin{equation}
\label{eqn:derivedcost}
\alpha_C(\vec{p}) = \sum_{i=1}^N p_i s_i (\vec{p})
\end{equation}
and scoring rules derived from convex cost functions with
\begin{equation}
\label{eqn:scoringrule}
s_i(\vec{p}) = \alpha_C(\vec{p}) - \sum_{j=1}^N \frac{\partial \alpha_C(\vec{p})}{\partial p_j} p_j +\frac{\partial \alpha_C(\vec{p})}{\partial p_i}.
\end{equation}
We show that there is a mapping between a mildly restricted class of convex cost function based markets and a mildly restricted class of strictly proper market scoring rules such that for every pair in the mapping, Equations~\ref{eqn:derivedcost} and~\ref{eqn:scoringrule} both hold.  Furthermore, we show that the markets satisfying these equations are equivalent in the sense described above.

\begin{theorem}
There is a one-to-one and onto mapping between the set of convex cost function based markets with strictly convex and differentiable potential functions $\alpha_C$ and the class of strictly proper, regular market scoring rules with differentiable scoring functions $\vec{s}$ such that for each pair in the mapping, Equations~\ref{eqn:derivedcost} and~\ref{eqn:scoringrule} hold.

Furthermore, each pair of markets in this mapping are equivalent when prices for all outcomes are positive, that is, the profit of a trade is the same in the two markets if the trade starts with the same market prices and results in the same market prices and the prices for all outcomes are positive before and after the trade.  Additionally, every price vector $\vec{p}$ achievable in the market scoring rule is achievable in the cost function based market.
\label{thm:msrequivalence}
\end{theorem}
\begin{proof}
We first show that the function $\alpha_C$ in Equation~\ref{eqn:derivedcost} is strictly convex and differentiable and the scoring rule in Equation~\ref{eqn:scoringrule} is regular, strictly proper and differentiable.  We then show that Equations~\ref{eqn:derivedcost} and~\ref{eqn:scoringrule} are equivalent.  Finally, we  show the equivalence between the two markets.

Consider the function $\alpha_C$ in Equation~\ref{eqn:derivedcost}.  Since we have assumed that $s_i$ is differentiable for all $i$, $\alpha_C$ is differentiable too.  Additionally, it is known that a scoring rule is strictly proper only if its expected value is strictly convex~\cite{Gneiting:07}, so $\alpha_C$ is strictly convex.

Consider the scoring rule defined in Equation~\ref{eqn:scoringrule}. By Theorem 1 of~\citet{Gneiting:07}, a regular scoring rule $s_i(\vec{p})$ is strictly proper if and only if there exists a strictly convex function $G(\vec{p})$ such that 
\begin{equation}
s_i(\vec{p}) = G(\vec{p}) - \sum_{j=1}^N  p_j \dot{G}_j(\vec{p}) + \dot{G}_i(\vec{p}) ,
\label{eqn:grform}
\end{equation}
where $G_j(\vec{p})$ is any subderivative of $G$ with respect to $p_j$ (if $G$ is differentiable, $\dot{G}_j = \partial G(\vec{p}) / \partial p_j$).  This immediately implies that the scoring rule defined in Equation~\ref{eqn:scoringrule} is a regular strictly proper scoring rule since $\alpha(\vec{p})$ is strictly convex.  We will see below that $s_i$ is also differentiable.

It is easy to see that Equation~\ref{eqn:scoringrule} implies Equation~\ref{eqn:derivedcost}.  Suppose Equation~\ref{eqn:scoringrule} holds.  Then
\begin{eqnarray*}
\sum_{i=1}^N p_i s_i (\vec{p})
&=& \sum_{i=1}^N p_i \!\left(\!\alpha_C(\vec{p}) - \sum_{j=1}^N \frac{\partial \alpha_C(\vec{p})}{\partial p_j} p_j +\frac{\partial \alpha_C(\vec{p})}{\partial p_i}\!\right)\!
\\
&=& \alpha_C(\vec{p}) ~.
\end{eqnarray*}
This also shows that $s_i$ is differentiable for all $i$, since the derivative of $\alpha_C$ is well-defined at all points and
\[
\frac{\partial \alpha_C(\vec{p})}{\partial p_i}
= s_i(\vec{p}) + \sum_{i=1}^N \frac{\partial s_i(\vec{p})}{\partial p_i} ~.
\]


To see that Equation~\ref{eqn:derivedcost} implies Equation~\ref{eqn:scoringrule}, suppose that Equation~\ref{eqn:derivedcost} holds.  We know that the scoring rule $\vec{s}$ can be expressed as in Equation~\ref{eqn:grform} for some function $G$.  For this particular $G$,
\[
\alpha_C(\vec{p}) 
= \sum_{i=1}^N p_i \left(G(\vec{p}) -\sum_{j=1}^N p_j \dot{G}_j(\vec{p}) +  \dot{G}_i(\vec{p}) \right)
= G(\vec{p}) ~.
\]
Since $G(\vec{p}) = \alpha_C(\vec{p})$ and $\alpha_C$ is differentiable (meaning that $\partial \alpha_C / \partial p_i$ is the only subderivative of $\alpha_C$ with respect to $p_i$), this implies Equation~\ref{eqn:scoringrule}.

We have established the equivalence between Equations~\ref{eqn:derivedcost} and~\ref{eqn:scoringrule}. We now show that a trader gets exactly the same profit for any realized outcome in the two markets if the market prices are positive.

Suppose in the cost function based market a trader changes the outstanding shares from $\vec{q}$ to $\vec{q}\,'$. This trade changes the market price from $\vec{p}(\vec{q})$ to $\vec{p}(\vec{q}\,')$.  If outcome $i$ occurs, the trader's profit is
\begin{eqnarray}
\lefteqn{(q_i' -q_i) - \left(C(\vec{q}\,')-C(\vec{q})\right)}
\nonumber\\
&=& (q_i' -q_i) - \left(\sum_{j=1}^N p_j(\vec{q}\,') q_j' - \alpha_C(\vec{p}(\vec{q}\,'))\right)
\nonumber\\
&& + \left(\sum_{j=1}^N p_j(\vec{q}) q_j - \alpha_C(\vec{p}(\vec{q}))\right)
\nonumber\\
&=& \left(q_i'- \sum_{j=1}^N p_j(\vec{q}\,') q_j' +\alpha_C(\vec{p}(\vec{q}\,'))\right)
\nonumber\\
&& -\left(q_i- \sum_{j=1}^N p_j(\vec{q}) q_j +\alpha_C(\vec{p}(\vec{q}))\right).
\label{eqn:profit}
\end{eqnarray}
From Lemma~\ref{lem:costrepresentation}, we know that $\vec{p}(\vec{q})$ is the optimal solution to the convex optimization $\max_{\vec{p} \in \simp_N} \left(\sum_{i=1}^N p_i q_i - \alpha_C(\vec{p})\right)$.
The Lagrange function of this optimization problem is
\[
L =  \left(\sum_{i=1}^N p_i q_i - \alpha_C(\vec{p})\right) - \lambda (\sum_{i=1}^N p_i -1)+\sum_{i=1}^N \mu_i p_i.
\]
Since $\vec{p}(\vec{q})$ is optimal, the KKT conditions require that $\partial L / \partial p_i =0$, which implies that for all $i$,
\begin{equation}
\label{eqn:foc}
q_i = \frac{\partial \alpha_C(\vec{p}(\vec{q}))}{\partial p_i(\vec{q})}+\lambda(\vec{q}) - \mu_i(\vec{q}) , 
\end{equation}
where $\mu_i(\vec{q}) \geq 0$ and $\mu_i(\vec{q})p_i(\vec{q}) =0$. Plugging (\ref{eqn:foc}) into (\ref{eqn:profit}), we have 
\begin{align}
&(q_i' -q_i) - \left(C(\vec{q}\,')-C(\vec{q})\right)
\nonumber \\
&=\left(\! \frac{\partial \alpha_C(\vec{p}(\vec{q}\,'))}{\partial p_i(\vec{q}\,')} -\sum_{j=1}^N  p_j(\vec{q}\,')\frac{\partial \alpha_C(\vec{p}(\vec{q}\,'))}{\partial p_j(\vec{q}\,')} + \alpha_C(\vec{p}(\vec{q}\,')) - \mu_i(\vec{q}\,')\!\right)
\nonumber \\
&\quad -\left(\frac{\partial \alpha_C(\vec{p}(\vec{q}))}{\partial p_i(\vec{q})} -\sum_{j=1}^N  p_j(\vec{q})\frac{\partial \alpha_C(\vec{p}(\vec{q}))}{\partial p_j(\vec{q})} + \alpha_C(\vec{p}(\vec{q})) - \mu_i(\vec{q})\right)
\nonumber \\
&=\left(s_i(\vec{p}(\vec{q}\,')) - \mu_i(\vec{q}\,')\right) - \left(s_i(\vec{p}(\vec{q})) - \mu_i(\vec{q})\right).
\label{eqn:costfuncprofit}
\end{align}
When $p_i(\vec{q}) >0$ and $p_i(\vec{q}\,') >0$,  $\mu_i(\vec{q}) = \mu_i(\vec{q}\,') = 0$. In this case, the profit of the trader in the cost function based market is the same as that in the market scoring rule market when he changes the market probability from $\vec{p}(\vec{q})$ to $\vec{p}(\vec{q}\,')$.  

Finally, observe that using the cost function based market it is possible to achieve any price vector $\vec{r}$ with finite scores $s_i(\vec{r})$ by setting $q_i = s_i(\vec{r})$ for all $i$.  By Lemma~\ref{lem:costrepresentation}, for this setting of $\q$, $p(\q)$ is the vector $\p$ that maximizes $\sum_{i=}^N p_i s_i(\vec{r}) - \sum_{i=1}^N p_i s_i(\vec{p})$.  Since $\vec{s}$ is strictly proper, this is maximized at $\vec{p} = \vec{r}$.  Since $\vec{s}$ is regular, this implies that it is possible to achieve any prices in the interior of the probability simplex using the cost function based market (and any prices $\p$ on the exterior as long as $s_i(\p)$ is finite for all $i$).
\end{proof}

\subsection{Convex Cost Functions and FTRL}
\label{sec:ftrlconnection}

Consider a prediction market with a convex cost function represented as $C(\q) = \sup_{\vec{p}\in \simp_N} \left(\sum_{i=1}^N p_i q_i - \alpha (\vec{p})\right)$ and the corresponding learning algorithm with weights $w_{i,t} = p_i(-\epsilon \vec{\Loss}_{t-1})$.  (Recall that $\vec{\Loss}_{t-1} = \langle \Loss_{1,t-1},\cdots,\Loss_{N,t-1} \rangle$ is the vector of cumulative losses at time $t-1$.)  By Lemma~\ref{lem:costrepresentation}, the weights chosen at time $t$ are those that maximize the expression $- \epsilon \sum_{i=1}^N w_i \Loss_{i,t-1} - \alpha (\vec{w})$, or equivalently, those that minimize the expression
\[
\sum_{i=1}^N w_i \Loss_{i,t-1} + \frac{1}{\epsilon} \alpha (\vec{w}) ~.
\]

This expression is of precisely the same form as Equation~\ref{eqn:ftrl}, with $\alpha$ playing the role of the regularizer and $\epsilon$ controlling the trade-off between the regularizer and the empirical loss.  This implies that every convex cost function based prediction market can be interpreted as a Follow the Regularized Leader algorithm with a convex regularizer!  By applying Theorem~\ref{thm:mainreduction} and Lemma~\ref{lem:worstcaseloss}, we can easily bound the regret of the resulting algorithm as follows.
\begin{theorem}
  Let $C$ be any valid convex cost function yielding $\phi$-stable prices, and let $\alpha_C$ be the penalty function associated with $C$.  Let $\alg$ be the expert learning algorithm with weights as in Equation~\ref{eqn:weights} with $\epsilon = \sqrt{2  \sup_{\vec{p},\vec{p}\,' \in \simp_N} (\alpha_C(\vec{p}) - \alpha_C(\vec{p}\,'))/(\phi T)}$.  Then for any sequence of expert losses $\loss_{i,t} \in [0,1]$ over $T$ time steps, 
\[ 
\Loss_{\alg,T} - \min_{i \in \{1,\cdots,N\}} \Loss_{i,T} 
\leq 
\sqrt{2 T \phi \sup_{\vec{p},\vec{p}\,' \in \simp_N} \left(\alpha_C(\vec{p}) - \alpha_C(\vec{p}\,')\right)} ~.
\]
\label{thm:ftrlreduction}
 \end{theorem} 
This bound is very similar to the bound for FTRL given in Equation~\ref{eqn:lambdabound}, with $\phi$ playing the role of $\lambda$.

The connections we established in the previous section imply that every strictly proper market scoring rule can also be interpreted as a FTRL algorithm, now with a \emph{strictly} convex regularizer. Conversely, any FTRL algorithm with a differentiable and strictly convex regularizer can be viewed as choosing weights at time $t$ to minimize the quantity
\[
\sum_{i=1}^N w_i \left( \epsilon \Loss_{i,t-1} + s_i (\vec{w})\right)
\]
for a strictly proper scoring rule $\vec{s}$.  Perhaps it is no surprise that the weight updates of FTRL algorithms can be framed in terms of proper scoring rules given that proper scoring rules are commonly used as loss functions in machine learning~\cite{BSS05,RW09} and FTRL has previously been connected to Bregman divergences~\cite{SS07,HK08,H09} which are known to be related to scoring rules~\cite{Gneiting:07}.

This connection hints at why market scoring rules and convex cost function based markets may be able to obtain accurate estimates of probability distributions in practice.  Both types of markets are essentially \emph{learning} the distributions by treating market trades as training data.  Beyond that, both markets correspond to well-understood learning algorithms with stable weights and guarantees of no regret.

\subsection{Relation to the SCPM}
\label{sec:scpm}

\citet{ADPWY09} present another way of describing convex cost function based prediction markets, which they call the Sequential Convex Pari-Mutuel Mechanism (SCPM).  The SCPM is defined in terms of limit orders instead of market prices, but the underlying mathematics are essentially the same.  In the SCPM, traders specify a maximum quantity of shares that they would like to purchase and a maximum price per share that they are willing to spend.  The market then decides how many shares of the trade to accept by solving a convex optimization problem.

\citet{ADPWY09} show that for every SCPM, there is an equivalent convex cost function based market.  For each limit order, the number of shares accepted by the market maker in the SCPM is the minimum of the number of shares requested by the trader and the number of shares that it would take to drive the market price of the shares in the corresponding cost function based market to the limit price of the trader.  Thus our results imply that any SCPM mechanism can also be interpreted as a Follow the Regularized Leader algorithm for learning from expert advice.

We remark that \citet{ADPWY09} also describe an interpretation of the SCPM in terms of convex risk measures and suggest that the associated penalty function is related to the underlying problem of learning the distribution over outcomes.  However, their interpretation is very different from ours.  They view the penalty function as characterizing ``the market maker's commitment to learning the true distribution'' since it impacts both the worst case market maker loss and the willingness of the market maker to accept limit orders.  On the contrary, we view the penalty function as a regularizer necessary to make the market prices stable.

\section{Example: The Quadratic MSR and Online Gradient Descent}

In the previous section we described the relationship between market scoring rules, cost function based markets with convex cost functions, and Follow the Regularized Leader algorithms.  We discussed how the Logarithmic Market Scoring Rule can be represented equivalently as a cost function based market, and how it corresponds to Weighted Majority in the expert learning setting.  In this section, we illustrate the relationship through another example.  In particular, we show that the Quadratic Market Scoring Rule can be written equivalently as a cost function based market (namely the Quad-SCPM of \citet{ADPWY09}).  We then show that this market corresponds to the well-studied online gradient descent algorithm in the learning setting and give a bound on the regret of this algorithm using Theorem~\ref{thm:ftrlreduction}.

The Quadratic Market Scoring Rule (QMSR) is the market scoring rule corresponding to the quadratic scoring function in Equation~\ref{eqn:qsr}.  As was the case in the LMSR, the parameters $a_1, \cdots, a_N$ do not affect the prices or payments of this market.  As such, we assume that $a_i = 0$ for all $i$.

Theorem~\ref{thm:msrequivalence} implies that we can construct a cost function based market with equivalent payoffs to the QMSR whenever prices are nonzero using the cost function
\begin{eqnarray*}
  C(\vec{q})
&=& \sup_{\vec{p}\in \simp_N} \left(\sum_{i=1}^N p_i q_i - \sum_{i=1}^N p_i 
b\left(2p_i - \sum_{i=1}^N p_i^2 \right) \right)
\\
&=& \sup_{\vec{p}\in \simp_N} \left(\sum_{i=1}^N p_i q_i - b \sum_{i=1}^N p_i^2 \right) ~.
\end{eqnarray*}
This is precisely the cost function associated with the Quad-SCPM market with a uniform prior, which was previously known to be equivalent to the QMSR when prices are nonzero~\cite{ADPWY09}.  The worst case loss of the market maker in both markets is $b (N-1)/N$.

Following the argument in Section~\ref{sec:ftrlconnection}, this market corresponds to the FTRL algorithm with regularizer $\eta = 1/b$ and $\Reg(\vec{w}) = \sum_{i=1}^N w_i^2$.  It has been observed that using FTRL with a regularizer of this form is equivalent to online gradient descent~\cite{HAK07,H09}.  Thus we can use Theorem~\ref{thm:ftrlreduction} to show a regret bound for gradient descent. 

We first show that the Quad-SCPM prices are $\phi$-stable for $\phi = (N^2-1)/(2b) < N^2/(2b)$.  (See Appendix~\ref{app:quad} for details.)  We can therefore apply Theorem~\ref{thm:ftrlreduction} using $\phi = N^2/(2b)$ and $\sup_{\vec{p},\vec{p}\,' \in \simp_N} \left(\alpha(\vec{p}) - \alpha(\vec{p}\,')\right) = b(N-1)/N < b$ to see that for gradient descent,
\[
\Loss_{GD,T} - \min_{i \in \{1,\cdots,N\}} \Loss_{i,T}
\leq N\sqrt{T} ~.
\]
This matches the known regret bound for general gradient descent applied to the experts setting~\cite{Z03}.

\section{Discussion}

We have demonstrated the elegant mathematical connection between market scoring rules, cost function based prediction markets, and no-regret learning.  This connection is thought-provoking on its own, as it yields to new interpretations of well-known prediction market mechanisms.  The interpretation of the penalty function as a regularizer can shed some light on which market scoring rule or cost function based market is best to run under different assumptions about traders.

Additionally, this connection has the potential to be of use in the design of new prediction market mechanisms and learning algorithms.  In recent years there has been an interest in finding ways to tractably run market scoring rules over combinatorial or infinite outcome spaces~\cite{CGP08,GCP09,CFLPW08}.  For example, a market maker might wish to accept bets over permutations (``horse A will finish the race ahead of horse B''), Boolean spaces (``either a Democrat will win the 2010 senate race in Delaware or a Democrat will win in North Dakota''), or real numbers (``Google's revenue in the first quarter of 2010 will be between $\$x$ and $\$y$''), in which case simply running a naive implementation of an LMSR (for example) would be infeasible.  As mentioned above, by exploiting the connection between Weighted Majority and the LMSR, \citet{CFLPW08} showed that an extension of the Weighted Majority algorithm to permutation learning~\cite{HW09} could be used to approximate prices in an LMSR over permutations.  Given our new understanding of the connection between markets and learning and the growing literature on no-regret algorithms for large or infinite sets of experts~\cite{HP05}, it seems likely that similar learning-based techniques could be developed to calculate market prices for other types of large outcome spaces too.

\bibliographystyle{plainnat}
{\small{\bibliography{marketsml}}}

\appendix

\section{Proof of Lemma 2}
\label{app:pricingdiffbound}

Fix the vectors $\q$ and $\r$.  Let $\u(s) = \q + s \r$. Similarly to how we have defined $D_{i,j}(\vec{t})$, define
\[
D_{i}(x) = 
\begin{cases}
\left. \frac{\partial p_i(\vec{u}(s))}{\partial s} \right|_{s=x} 
& \textrm{if  $\frac{\partial p_i(\vec{u}(s))}{\partial s} $ is defined at $x$,}
\\
0  
& \textrm{otherwise.}
\end{cases}
\]
Using $D_{i}$ in place of the derivative  allows us to integrate over the derivative even when it is not defined at single points.  For any point at which the derivatives are defined, we have
\[
\frac{\partial C(\u(s))}{\partial s}
= \sum_{i=1}^N \frac{\partial C(\u(s))}{\partial u_i(s)} 
       \frac{\partial u_i(s)}{\partial s}
= \sum_{i=1}^N p_i(\u(\cdot)) r_i ~.
\]
Applying the fundamental theorem of calculus, we have that
\begin{eqnarray*}
\lefteqn{C(\q+\r) - C(\q) }
\\
&=& \int_0^1 \left. \frac{\partial C(\u(s))}{\partial s} \right|_{s=x} dx
= \int_0^1 \sum_{i=1}^N p_i(\u(x)) r_i  \ dx
\\
&=& \int_0^1 \sum_{i=1}^N 
\left(p_i(\u(0)) + \int_0^x 
D_{i}(y)
dy \right) r_i  \ dx
\\
&=& 
\sum_{i=1}^N p_i(\q) r_i
+ \int_0^1 \int_0^x \sum_{i=1}^N r_i 
D_{i}(y)
\  dy \  dx ~.
\end{eqnarray*}
Rearranging terms, this gives us that
\[
C(\q+\r) - C(\q) - \sum_{i=1}^N p_i(\q) r_i
= \int_0^1 \int_0^x \sum_{i=1}^N r_i 
D_{i}(y)
\ d y \ d x.
\]

To prove the lemma, it is sufficient to bound the absolute value of the expression on the right.   This is where the $\phi$-stability of the prices comes into play. At any point where the derivatives are defined,
\begin{eqnarray*}
\sum_{i=1}^N \abs{ \frac{\partial p_i(\u(s))}{\partial s} } 
&=& \sum_{i=1}^N \abs{ \sum_{j=1}^N \frac{\partial p_i(\q)}{\partial
    q_j} r_j}
\leq \epsilon \sum_{i=1}^N \sum_{j=1}^N
\abs{\frac{\partial p_i(\q)}{\partial
    q_j}} 
\\
&\leq& \epsilon \phi ~.
\end{eqnarray*}
Since we have assumed that the prices are piecewise differentiable this implies that
\begin{eqnarray*}
\lefteqn{\abs{
\int_0^1 \int_0^x \sum_{i=1}^N r_i 
D_{i}(y)
\ dy \ dx}}
\\
&\leq& 
\epsilon \int_0^1 \int_0^x \sum_{i=1}^N 
\abs{D_{i}(y)}
\ dy \ dx 
\\
&\leq&
\epsilon \int_0^1 \int_0^x 
\epsilon \phi
\ dy \ dx
= \epsilon^2 \phi 
 \int_0^1 x \ dx
= \frac{\epsilon^2 \phi}{2} ~.
\end{eqnarray*}
This bounds the absolute value of the right hand side of the equation above and proves the lemma.
\qed

\section{Proof of Lemma 3}
\label{app:lmsrderiv}

For every $i$ and every $j \neq i$, we have
\begin{eqnarray*}
\frac{\partial p_i(\q)}{\partial q_i}
&=& \frac{\partial}{\partial q_i} \frac{ \e^{q_{i}/b}}{\sum_{j=1}^N \e^{q_{j}/b}}
\\
&=& \frac{1}{b} 
\frac{e^{q_i/b} \sum_{j=1}^N \e^{q_{j}/b} - \left(e^{q_i/b}\right)^2}
{\left(\sum_{j=1}^N \e^{q_{j}/b}\right)^2}
\\
&=& \frac{1}{b} \left(p_i(\q) - p_i(\q)^2\right) 
= \frac{1}{b} p_i(\q) \sum_{j \neq i} p_j(\q)
\end{eqnarray*}
and
\begin{eqnarray*}
\frac{\partial p_i(\q)}{\partial q_j}
&=& \frac{\partial}{\partial q_j} \frac{ \e^{q_{i}/b}}{\sum_{j=1}^N \e^{q_{j}/b}}
= \frac{1}{b}
\frac{- \e^{q_i/b} \e^{q_j/b}}
{\left(\sum_{j=1}^N \e^{q_{j}/b}\right)^2}
\\
&=& - \frac{1}{b} p_i(\q) p_j(\q) ~.
\end{eqnarray*}
Thus we have
\begin{eqnarray*}
\sum_{i=1}^N \sum_{j=1}^N \abs{\frac{\partial p_i(\q)}{\partial q_j}}
&=& \sum_{i=1}^N \abs{\frac{\partial p_i(\q)}{\partial q_i}} + 
\sum_{i=1}^N \sum_{j \neq i} \abs{\frac{\partial p_i(\q)}{\partial q_j}}
\\
&=& \frac{2}{b} 
\sum_{i=1}^N \sum_{j \neq i} p_i(\q) p_j(\q) ~.
\end{eqnarray*}
We would like to find the prices that maximize this quantity.  Dropping the argument $\q$ to simply notation, this is equivalent to solving a simple optimization problem:
\[
\maximize_{\p \in \simp_N} \ \ \frac{2}{b}  \sum_{i=1}^N \sum_{j \neq i} p_i p_j .
\]
It is straight-forward to show (e.g., using the KKT conditions) that this expression is maximized when the prices are equal across all securities, so $p_i = 1/N$ for all $i$.  Then
\[
\sum_{i=1}^N \sum_{j=1}^N \abs{\frac{\partial p_i(\q)}{\partial q_j}}
\leq \frac{2}{b} \frac{N(N-1)}{N^2}
\leq \frac{2}{b} ~.
\]
\qed

\section{Stability of Quad-SCPM Prices}
\label{app:quad}
The cost function of Quad-SCPM can be written as 
\[
C(\vec{q})= \sup_{\vec{p}\in\simp_N} \left(\sum_{i=1}^N p_i q_i - b \sum_{i=1}^N p_i^2 \right) ~.
\]
By Lemma~\ref{lem:costrepresentation}, the price function of Quad-SCPM is defined by the optimal solution to the optimization problem in this function. 
The Lagrange function corresponding to the constrained optimization problem is
\[
L =  \left(\sum_{i=1}^n p_i q_i - b\sum_{i=1}^Np_i^2\right) 
- \lambda \left(\sum_{i=1}^N p_i -1\right)+\sum_{i=1}^N \mu_i p_i.
\]
Because the objective function is strictly concave in $\vec{p}$, there is a unique optimal solution that satisfies the KKT conditions: $\partial L/\partial p_i=0~\forall i$, $\sum_{i=1}^N p_i =1$, $p_i\geq0~\forall i$, $\mu_i \geq 0~\forall i$, and $\mu_i p_i =0~\forall i$. Thus, the price function is defined by
\begin{equation}\label{quad}
\begin{cases}
p_i = \frac{1}{N} + \frac{q_i+\mu_i}{2b} - \frac{\sum_{j=1}^N (q_j+\mu_j)}{2bN}~\forall i\\
p_i \mu_i =0~\forall i\\
p_i\geq 0~\forall i\\
\mu_i \geq 0~\forall i. 
\end{cases}
\end{equation}

When $1/N + q_i/(2b) - \sum_{i=1}^N q_i / (2bN) > 0$, $\mu_i=0$ and $p_i >0$ for all $i$ and the price function is
\[
p_i(\vec{q}) = \frac{1}{N} + \frac{q_i}{2b} - \frac{\sum_{j=1}^N q_j}{2bN},
\]
which is the same as that of QMSR. In this case, we have
\[
\frac{\partial p_i}{\partial q_i}=\frac{1}{2b} - \frac{1}{2bN} ~\forall i 
\quad \textrm{and} \quad 
\frac{\partial p_i}{\partial q_j}= - \frac{1}{2bN}~\forall j\neq i ~.
\]

Consider the case in which the prices for some outcomes equal 0.  Given $\vec{q}$, let $M=\{m: p_m >0, \mu_m =0\}$ be the set of outcomes that have positive prices. Let $K=\{k: p_k=0, \mu_k >0\}$ be the set of the outcomes that have a positive $\mu_k$. Let $L = \{l: p_l =0, \mu_l=0\}$ be the set of outcomes for which both $p_l$ and $\mu_l$ are 0. 
Denote $\tilde{p} = \min_{m\in M} p_m$ and $\tilde{\mu} = \min_{k\in K} \mu_k$. For $0<|\epsilon|<\min \{\tilde{\mu}, 2b\tilde{p}(N-|K|-|L|)\}$, we consider the following cases: 
\squishlist

\item $L =\emptyset$ and $i \in M$

Consider changing $\vec{q}$ to $\vec{q}\,'$, where $q_i' = q_i + \epsilon$ and $q_j' = q_j~\forall j\neq i$. We can verify that the price function for $\vec{q}\,'$  is defined by
\begin{equation*}
\begin{cases}
p_i' = p_i + (\frac{1}{2b} -\frac{1}{2b(N-|K|)})\epsilon\\
p_j' = p_j -\frac{\epsilon}{2b(N-|K|)}&\forall j \in M, j\neq i\\ 
p_k' = p_k =0 &\forall k\in K\\
\mu_j' = \mu_j=0 &\forall j \in M\\
\mu_k' = \mu_k + \frac{\epsilon}{N-|K|} &\forall k\in K 
\end{cases}
\end{equation*}
where $\vec{p}$ and $\vec{\mu}$ are the prices and Lagrange multipliers for $\vec{q}$.  Hence, by the definition of derivatives, we have
\[
\frac{\partial p_i}{\partial q_i}=\frac{1}{2b} -\frac{1}{2b(N-|K|)} \quad \forall i \in M ~,
\]
\[
\frac{\partial p_j}{\partial q_i} = -\frac{1}{2b(N-|K|)} \quad \forall i, j \in M, i\neq j ~,
\]
\[
\frac{\partial p_k}{\partial q_i}= 0 \quad  \forall k \in K, i\in M ~.
\]

\item $L \neq \emptyset$ and $i \in M$

Consider changing $\vec{q}$ to $\vec{q}\,'$, where $q_i' = q_i + \epsilon$ and $q_j' = q_j~\forall j\neq i$. The new prices are defined by
\begin{equation*}
\begin{cases}
p_i' = p_i + (\frac{1}{2b} -\frac{1}{2b(N-|K|-|L|)})\epsilon\\
p_j' = p_j -\frac{\epsilon}{2b(N-|K|-|L|)}&\forall j \in M, j\neq i\\ 
p_k' = p_k =0 &\forall k\in K\\
p_l' = p_l =0 &\forall l\in L\\
\mu_j' = \mu_j=0 &\forall j \in M\\
\mu_k' = \mu_k + \frac{\epsilon}{N-|K|-|L|} &\forall k\in K \\
\mu_l' = \mu_l + \frac{\epsilon}{N-|K|-|L|} &\forall l \in L
\end{cases}
\end{equation*}
if $\epsilon >0$, and
\begin{equation*}
\begin{cases}
p_i' = p_i + (\frac{1}{2b} -\frac{1}{2b(N-|K|)})\epsilon\\

p_j' = p_j -\frac{\epsilon}{2b(N-|K|)}&\forall j \in M, j\neq i\\ 
p_k' = p_k =0 &\forall k\in K\\
p_l' = p_l - \frac{\epsilon}{2b(N-|K|)}&\forall l\in L\\
\mu_j' = \mu_j=0 &\forall j \in M\\
\mu_k' = \mu_k + \frac{\epsilon}{N-|K|} &\forall k\in K \\
\mu_l' = \mu_l =0 &\forall l \in L\\
\end{cases}
\end{equation*}
if $\epsilon <0$. 

From the above, we can see that the prices of all outcomes are continuous while we changing $q_i$ in its $\epsilon$-neighborhood. However, $\partial p_i/\partial q_i$, $\partial p_j/\partial q_i$, and $\partial p_l/\partial q_i$ are not defined at $\vec{q}$ for all $i, j \in M$, $i\neq j$, and $l\in L$, because the left and right derivatives do not equal. We only have  $\partial p_k/\partial q_i=0$ for all $k\in K$ and $i \in M$.  

\item $L \neq \emptyset$, $l \in L$

When changing $\vec{q}$ to $\vec{q}\,'$, where $q_l' = q_l + \epsilon$ and $q_j' = q_j~\forall j\neq l$, similar to the above case, the new optimal solution is different for $\epsilon >0$ and $\epsilon < 0$. In particular, when $\epsilon>0$,
\[
p_l' = p_l +(\frac{1}{2b} -\frac{1}{2b(N-|K|)})\epsilon , 
\]
\[
p_j' = p_j -\frac{\epsilon}{2b(N-|K|)}~\forall j \in M\cup L \textrm{ such that } j\neq l, 
\]
and $p_k'=p_k=0~\forall k \in K$.  
When $\epsilon<0$, $p_j'=p_j~\forall j$, $\mu_l' = \mu_l - \epsilon$, and $\mu_j' = \mu_j~\forall j\neq k$.  

We can see that the prices of all outcomes are continuous while we changing $q_l$ in  its $\epsilon$-neighborhood. But $\partial p_l/\partial q_l$ and $\partial p_j/\partial q_l$ are not defined at $\vec{q}$ for all $j \in M\cup L$, $j\neq l$, and $l\in L$, because the left and right derivatives do not equal. $\partial p_k/\partial q_l=0$ for all $k\in K$ and $l \in L$. 
 
\item $k \in K$

Consider changing $\vec{q}$ to $\vec{q}\,'$, where $q_k' = q_k + \epsilon$ and $q_j' = q_j~\forall j\neq k$, the new prices are defined by $p_j' = p_j~\forall j$, $\mu_k' = \mu_k - \epsilon$, and $\mu_j' = \mu_j~\forall j\neq k$. Thus, we have $\partial p_k/\partial q_k=0~\forall k \in K$, $\partial p_j/\partial q_k=0~\forall k, j \in K$, and $\partial p_j/\partial q_k=0~\forall j\in M, k\in K$. 
\squishend

The above shows that the price functions are continuous everywhere, but not differentiable everywhere. In particular, when $L$ is not empty (i.e., given $\vec{q}$ there exists some outcome $i$ such that both $p_i(\vec{q})$ and $\mu_i(\vec{q})$ are zero), some of the partial derivatives are not defined at $\vec{q}$. This corresponds to the second and third cases shown above. We further note that in these two cases, for any $i \in M$ and $l \in L$, any change in $q_i$ or $q_l$ will cause the set of $L$ to become empty, because either $p_l$ or $\mu_l$ will become positive for all $l\in L$. This means that the prices are differentiable almost everywhere, with the only exceptions at finite number of points when $L \neq \emptyset$.   

Because $0\leq |K|\leq N-1$, $0\leq D_{i,i}(t) \leq 1/(2b) - 1/(2bN)$ and $-1/(2b) \leq D_{i,j}(t) \leq 0$. We have 
\[
\sum_{i=1}^N \sum_{j=1}^N |D_{i,j}(t)| \leq N\left(\frac{1}{2b}-\frac{1}{2bN}\right)+N(N-1)\frac{1}{2b}=\frac{N^2-1}{2b}.
\]
\qed

\end{document}